\def\year{2017}\relax
\documentclass[letterpaper]{article}
\usepackage[utf8]{inputenc}
\usepackage{aaai17}
\usepackage{times}
\usepackage{helvet}
\usepackage{courier}
\usepackage{amssymb}
\usepackage{amsmath}
\usepackage{graphicx}
\newtheorem{theorem}{Theorem}
\newtheorem{lemma}{Lemma}
\newenvironment{proof}[1][Proof]{\textbf{#1.} }{\ \rule{0.5em}{0.5em}}
\def\VAT{AIVAT}
\def\leduc{Leduc hold'em}
\def\Leduc{Leduc Hold'em}
\def\texas{Texas hold'em}
\def\Texas{Texas Hold'em}
\def\etal{\textit{et al.}}
\newcommand{\ie}{{\it i.e., }}

\begin{document}

\title{AIVAT: A New Variance Reduction Technique for Agent Evaluation in Imperfect Information Games}
\author{Neil Burch, Martin Schmid, Matej Moravčík, Michael Bowling \\
Department of Computing Science \\
University of Alberta\\
\{nburch,mschmid,moravcik,mbowling\}@ualberta.ca
}
\maketitle

\begin{abstract}
Evaluating agent performance when outcomes are stochastic and agents
use randomized strategies can be challenging when there is limited
data available. The variance of sampled outcomes may make the simple
approach of Monte Carlo sampling inadequate. This is the case for
agents playing heads-up no-limit Texas hold'em poker, where
man-machine competitions have involved multiple days of consistent
play and still not resulted in statistically significant conclusions
even when the winner's margin is substantial. In this paper, we
introduce \VAT{}, a low variance, provably unbiased value assessment
tool that uses an arbitrary heuristic estimate of state value, as well
as the explicit strategy of a subset of the agents. Unlike existing
techniques which reduce the variance from chance events, or only
consider game ending actions, \VAT{} reduces the variance both from
choices by nature and by players with a known strategy. The resulting
estimator in no-limit poker can reduce the number of hands needed to
draw statistical conclusions by more than a factor of 10.
\end{abstract}

\section{Introduction}
Evaluating an agent's performance in stochastic settings can be hard.
Non-zero variance in outcomes means the game must be played multiple
times to compute a confidence interval that likely contains the true
expected value. Regardless of whether the variance arises from player
actions or from chance events, we might need to observe many samples
before we get a narrow enough interval to draw desirable conclusions.
In many situations, it is simply not feasible (e.g., when the
evaluation involves human participation) to simply observe more
samples, so we must turn to statistical techniques that use additional
information to help narrow the confidence interval.

This agent evaluation problem is commonly encountered in games, where
the goal is to estimate the expected performance difference between
players. For example, consider poker games. Poker is not only a
long-standing challenge problem for
AI~\cite{VN28,KollerPfeffer97,BillingsEtAl02} with annual
competitions~\cite{ZinkevichLittman06,ACPC}, but also a very popular
game played by an estimated 150 million players
worldwide~\cite{Economist07}. Heads-up no-limit Texas hold'em (HUNL)
is a particular variant of the game that has received considerable
attention in the AI community in recent years, including a ``Brains
vs. AI'' event pitting Claudico~\cite{Claudico}, a top HUNL computer
program, against professional poker players. That match involved
80,000 hands of poker, played over seven days, involving four poker
players, playing dozens of hours each. Despite Claudico losing by over
9 big blinds per 100 hands (a margin that is considered huge by poker
professionals)~\cite{Pokerfuse-Claudico}, the result is only on the
edge of statistical significance, making it hard to draw a conclusion
from this large investment of human time.

Previous techniques for variance reduction to achieve stronger
statistical conclusions in this setting have used two broad classes of
statistical techniques. Techniques like MIVAT~\cite{09white-mivat} use
the method of control variates with heuristic value estimates to
reduce the variance caused by chance events. The technique of
importance sampling over imaginary observations~\cite{08icml-ispoker}
takes a different approach, using knowledge of a player strategy to
evaluate multiple states given a single observation. Imaginary
observations can be used to reduce the variance caused by privately
observed chance events, as well as the player's randomly chosen choice
of whether to make any actions which would immediately end the game.

Techniques from the two classes can be combined, but are not
specifically designed to work together for the greatest reduction in
variance, and none of the techniques deal with the variance caused by
non-terminal action selection. Because good play in imperfect
information games generally requires randomised action selection,
ignoring action variance is an important shortcoming. We introduce the
action-informed value assessment tool (\VAT{}), an unbiased
low-variance estimator for imperfect information games which extends
the use of control variates to player actions, and makes explicit use
of imaginary observations to exploit knowledge of the game structure
and player strategies.

\section{Background}
This paper focuses on variance reduction when evaluating agents for
extensive form games, a class of imperfect information sequential
decision making problems. Formally, an extensive form game is a set of
of players $P$ and chance player $p_c$, a set of states $S$ described
as a history of actions from the initial state $\varnothing$, a set
$Z \subset S$ of terminal states, acting player $p(h) : S \setminus
Z \mapsto P \bigcup \{p_c\}$, player value functions $v_p(z) :
Z \mapsto \mathbf{R}$, and information partitions $\mathcal{I}_p$ of
$\{h \in S | p(h)=p\}$. We will say $h \sqsubset h'$ if a game in
state $h'$ was previously in state $h$, $h \sqsubseteq h'$ if
$h \sqsubset h'$ or $h=h'$, $A(h)$ is the set of valid actions at $h$,
and $h \cdot a$ is the successor state of $h$ that is reached by
making action $a$. For all states $h$ such that $p(h) = p_c$,
$\sigma_{p_c}(h,a)$ is the publicly known probability distribution
over possible chance outcomes at state $h$.

An information set $I \in \mathcal{I}_p$ describes a set of states
that player $p$ can not distinguish due to imperfect information of
the game state. Any player decision is therefore made at information
sets, not states. A behaviour strategy $\sigma_p(I,a)$ gives the
probability of player $p$ making decision $a$ at information set
$I$. The behaviour in a state is determined by the information set
$I$, so that $\forall h \in I$ $\sigma_p(h,a) = \sigma_p(I,a)$. We
will say the probability of reaching a state $h$ is $\pi(h)
= \Pi_{h' \cdot a} \sigma_{p(h')}(h',a)$.  It is also useful to
consider $\pi_p(h) = \Pi_{h' \cdot a \sqsubseteq h,
p(h')=p} \sigma_p(h',a)$, the probability of a player reaching state
$h$ if all other players play to reach $h$. This notation can be
extended so that for any set of players $T$, $\pi_{T}(h) = \Pi_{p \in
T} \pi_p(h)$.

When talking about estimating the value for players in a game, we are
trying to find the expected value $\mathop{\mathbb{E}}_z[v_p(z)]
= \sum_{z \in Z}\pi(z)v_p(z)$. An estimator $e(z)$ is said to be
unbiased if the expected value $\mathop{\mathbb{E}}_z[e(z)]
= \mathop{\mathbb{E}}_z[v_p(z)]$. Having an estimator be provably
unbiased is important because it is in some sense truthful: a player
can not appear to do better by changing their play to take advantage
of the estimation method.

\section{MIVAT and Imaginary Observations}
\VAT{} is an extension of two earlier techniques, MIVAT and importance
sampling over imaginary observations. MIVAT~\cite{09white-mivat} and
its precursor DIVAT~\cite{06zinkevich-divat} use value functions for a
control variate that estimates the expected utility given observed
chance events. Conceptually, the techniques subtract the expected
chance utility to get a lower variance value which mostly depends on
the player actions. For example, in poker, it is likely that good
hands end in positive outcomes and bad hands end in negative
outcomes. Starting with the observed outcome, we could subtract some
value for good hands and add a value for bad hands, and we would
expect the corrected value to have lower variance. If the expected
value of the correction terms is zero, we can use the lower variance
corrected value as an unbiased estimator of player value.

DIVAT requires a strategy for all players to generate value estimates
for states through self-play, which MIVAT generalised by allowing for
arbitrary value functions defined after chance events. MIVAT adds a
correction term for each chance event in an observed state. In order
to remain unbiased despite using an arbitrary value estimation
function $u(a)$, MIVAT uses a correction term of the form
$\mathop{\mathbb{E}}_a[u(a)] - u(o)$ for an observation with outcome
$o$. Computing this expectation requires us to know the probability
distribution that $o$ was drawn from, which is true in the case of
chance events as $\sigma_{p_c}$ is public knowledge. These terms are
guaranteed to have an expected value of zero, making the MIVAT value
(observed value plus correction terms) an unbiased estimate of player
value.  In a game like poker, MIVAT will account for the dealer giving
a player favourable or unfavourable cards, but not for lucky player
actions selected from a randomised strategy.

Imaginary observations with importance sampling~\cite{08icml-ispoker}
use knowledge of a player's strategy to compute an expected value of
multiple states given an observation of a single state. Due to
imperfect information, there may be many states which are all
guaranteed to have the same probability of the opponent making their
actions. If we consider importance sampling over these imaginary
observations, the opponent's probability of reaching the state cancels
out so we do not need the opponent's strategy. By taking an
expectation over a set of states for every observation, we get a lower
variance value.

There are two kinds of situations where we can use imaginary
observations. First, for any states $h$ where player $p$ could have
made an action $a$ which ends the game, we can add the imaginary
observation of the terminal state $h \cdot a$. For example, in poker
this lets us consider player $p$ folding to a bet they called or
raised, or calling a bet we folded to in the final round. Second,
because of the information partitions in imperfect information games,
there may be other states that have identical opponent
probabilities. In poker, this lets us consider all the states where
the public player actions are the same, the opponent private cards and
public board cards are the same, but player $p$ has different private
cards. Imaginary observations do not let us reduce the variance caused
by choosing non-terminal actions or the outcomes of publicly visible
chance events.

MIVAT and imaginary observations consider different information and
can be combined to get a value estimate with lower variance than
either technique used individually. Instead of using the terminal
value $v(z)$ for an imaginary observation $z$, we could use the MIVAT
value estimate given $z$. However, because neither technique has terms
which address the effect of non-terminal actions, we would never
expect this combination of techniques to produce a zero variance value
estimate.  Even with a ``perfect'' value function that correctly
estimates the expected value of a state and action for the players,
there would still be some variance in the value estimate due to the
random action selection by players.

\section{\VAT{}}
Conceptually, \VAT{} combines the chance correction terms of MIVAT
with imaginary observations across private information, along with new
MIVAT-like correction terms for player actions. The \VAT{} estimator
is the sum of a base value using imaginary observations, plus
imaginary observation correction terms for both player actions and
chance events. Roughly speaking, moving backwards through the choices
in an observed game, the \VAT{} correction terms are constructed in a
fashion that shifts an estimate of the expected value after a choice
was made towards an estimate of the expected value before the choice.

Because imaginary observations with importance sampling provides an
unbiased estimate of the expected value of the players, and the
MIVAT-like terms have an expected value of zero, \VAT{} is also an
unbiased estimator of the expected player value. Furthermore, with
well-structured games, ``perfect'' value functions, and knowledge of
all player strategies, we could see zero variance: the imaginary
observation values and the correction terms would sum to the expected
player value, regardless of the observed game.

Figure~\ref{fig:vat-diagram} gives a high level overview of MIVAT,
imaginary observations, and \VAT{}. In this example, we are interested
in the expected value for player 1, and know player 1's strategy. We
use an observation of one hand of \leduc{} poker, a small synthetic
game constructed for artificial intelligence
research~\cite{05uai-bayes}. \leduc{} is a two round game with one
private card for each player, and one publicly visible board card that
is revealed after the first round of player actions. In the example,
player 1 is dealt \textbf{Q$\spadesuit$} and player 2 is
dealt \textbf{K$\spadesuit$}. Player 1 makes the \textbf{check} action
followed by a player 2 \textbf{check} action. The public board card is
revealed to be \textbf{J$\heartsuit$}. After the round two
actions \textbf{check}, \textbf{raise}, \textbf{call}, player 1 loses
5 chips.

\begin{figure}[!ht]
\includegraphics[width=0.99\hsize]{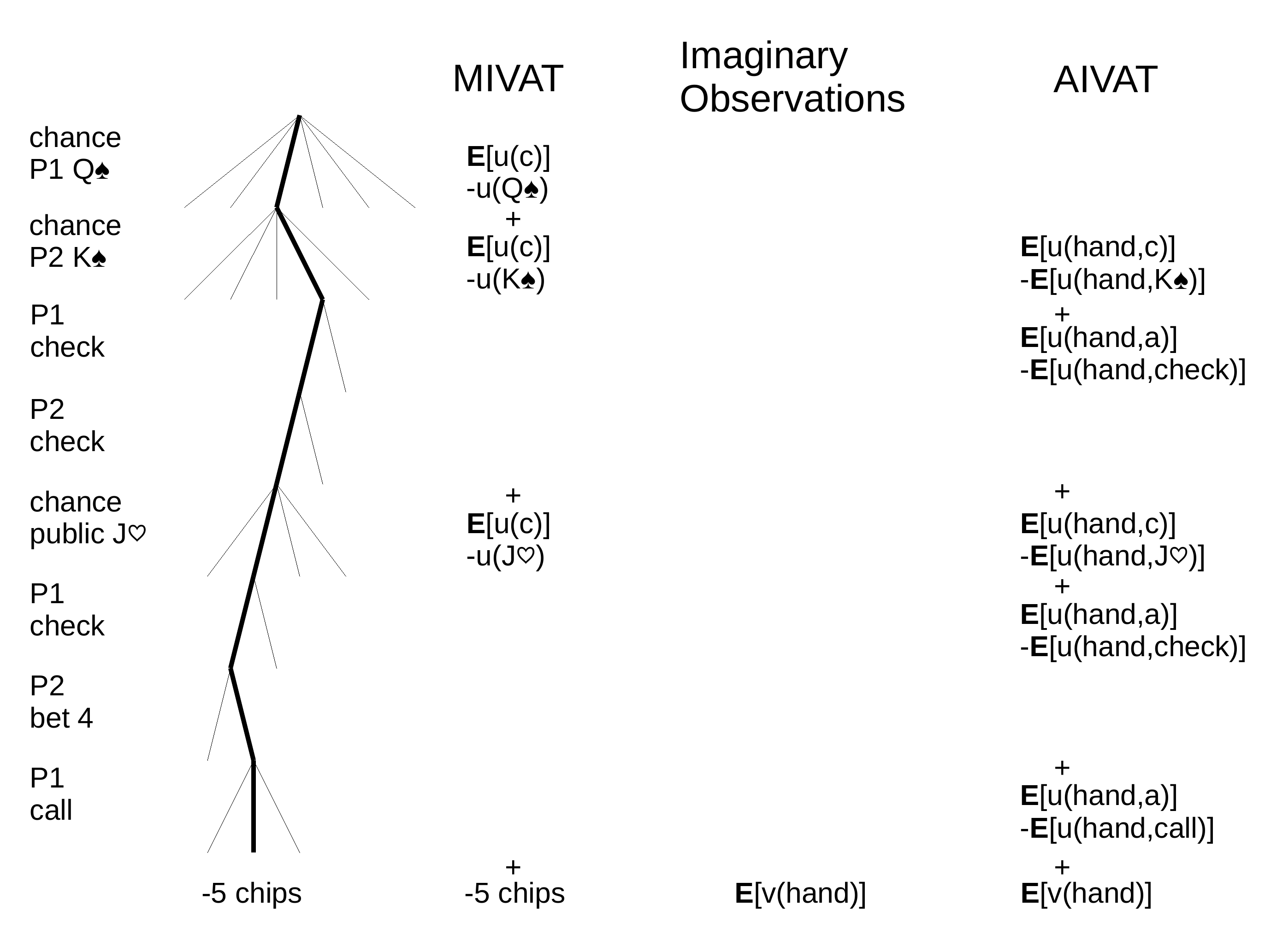}
\caption{Comparison of MIVAT, imaginary observations, and \VAT{}}
\label{fig:vat-diagram}
\end{figure}

\subsection{\VAT{} Correction Terms}
We start by describing the correction terms added for chance events
and actions. Given information about a player's strategy, we can treat
that player's choice events as chance events and construct MIVAT-like
correction terms for them. The player strategy also allows imaginary
observations considering alternative histories with identical opponent
probabilities, so we can compute an expectation over a set of
compatible histories rather than using the single observed outcome.

The correction term at a decision point will be the expectation across
all compatible histories of the expected value before a choice, minus
the value after the observed choice. As with MIVAT, the values are
estimated using an arbitrary fixed value function to estimate the
value after every decision. Value estimates which more closely
approximate the true expected value will result in greater variance
reduction.

To consider imaginary observations, we need at least one player for
which we know the know the strategy. Let $P_a$ be a non-empty set of
players, including $p_c$, such that $\forall p \in P_a$ we know
$\sigma_p$, and $P_o = P \setminus P_a$ be the set of opponent players
for which we do not know the strategy. If $P_a = \{p_c\}$ then \VAT{}
would be identical to MIVAT. We must also partition the states into
the sets we can evaluate given an observation of a completed game. Let
$\mathcal{H}$ be a partition of states $\{ h | p(h) \in P_a\}$ such
that $\forall H \in \mathcal{H}$ and $\forall h,h' \in H$,
\begin{enumerate}
\item $\forall p \in P_o$ $\forall \sigma_p$ $\pi_p(h) =
  \pi_p(h')$. For example, this can be enforced by requiring $h$ and
  $h'$ to pass through the same sequence of player $p$ information
  sets and make the same actions at those information sets.
\item $h \not \sqsubset h'$. This implies a uniqueness property, where
  for any terminal $z$, $\{ h'' | h'' \sqsubset z, h'' \in H\}$ is
  either empty or a singleton.
\item We will extend the actions so that $A'(h) = \bigcup_{h''
  \in H} A(h'')$ and let $\sigma(h,a)=0$ $\forall a \in
  A'(h) \setminus A(h)$. Because $A'(h)=A'(h')$ we will say
  $A(H)=A'(h)$.
\end{enumerate}

Similar to MIVAT, we need value functions that give an estimate of the
expected value after an action. Let there be arbitrary functions
$u_h(a):A'(h)\mapsto\mathbf{R}$ for each state $h$ where $p(h) \in
P_a$. Say we have seen a terminal state $z$. Consider a part $H \in
\mathcal{H}$. If $\nexists h \in H$ such that $h \sqsubset z$, then
the correction term $k_H(z)=0$. Otherwise, property 2 of $\mathcal{H}$
implies there is a unique observed action $a_\text{O}$ such that
$h \cdot a_\text{O} \sqsubseteq z, h \in H, a_\text{O} \in A(h)$, and
the correction term is
\begin{align*}
k_H(z)=
\frac{\sum_{a \in A(H)} \sum_{h \in H} \pi_{P_a}(h \cdot a)u_h(a)}
{\sum_{h \in H} \pi_{P_a}(h)} \\
- \frac{\sum_{h \in H}\pi_{P_a}(h \cdot a_\text{O})u_h(a_\text{O})}
{\sum_{h \in H} \pi_{P_a}(h \cdot a_\text{O})}
\end{align*}
\VAT{} uses the sum of $k_H(z)$ across all $H \in \mathcal{H}$.

\subsection{\VAT{} Base Value}
The \VAT{} correction terms have an expected value of zero, and are
not a value estimate by themselves. They must be combined with an
unbiased estimate of player value. For improved variance reduction,
the form of the correction terms must match the choice of base value
estimate.

To see how the terms match, consider a simplified version of \VAT{}
where the final correction term for a terminal state $h \cdot o$ has
the form $\mathop{\mathbb{E}}_a[u_h(a)] - u_h(o)$. Ideally, we would
like the value estimate for $h \cdot a$ to be $u_h(a)$. The value
estimate plus the correction term will then have the same value
$\mathop{\mathbb{E}}_a[u_h(a)]$ for all actions at $h$, resulting in
zero variance.

For the \VAT{} correction terms, the correct choice is to use
imaginary observations of all possible private information for players
in $P_a$, as in ``Example 3: Private Information'' of the paper by
Bowling~\etal{}~\cite{08icml-ispoker}. In poker, it corresponds to
evaluating the game with all possible private cards, weighted by the
likelihood of holding the cards given the observed game. For
completeness, we formally describe the particular instance of this
existing estimator using the notation of this paper.

Given the correction term partition $\mathcal{H}$ of player $P_a$
states, we construct a matching partition $\mathcal{W}$ of terminal
states such that $\forall W \in \mathcal{W}$ and $\forall z,z' \in W$,
\begin{itemize}
\item $\forall p \in P_o$ $\forall \sigma_p$ $\pi_p(z) =
  \pi_p(z')$.
\item a player in $P_a$ made an action in $z$ $\iff$ a player in $P_a$
  made an action in $z'$.
\item if a player in $P_a$ made an action in $z$, then for the longest
  prefix $h \sqsubset z$ and $h' \sqsubset z'$ such that $p(h) \in
  P_a$ and $p(h') \in P_a$, both $h$ and $h'$ are in the same part of
  $\mathcal{H}$.
\end{itemize}
The last two conditions on $\mathcal{W}$ ensure that the imaginary
observation estimate does not include terminal states that the
correction terms will also account for. This rules out a form of
double counting which would not produce a biased estimator, but would
increase the variance when using high quality estimates in the
correction terms.

If we observe a terminal state $z$, let $W \in \mathcal{W}$ be the
part such that $z \in W$. The base estimated value for player $p$ is
\begin{align*}
\frac{\sum_{z' \in W} \pi_{P_a}(z') v_p(z')}{\sum_{z' \in W}\pi_{P_a}(z')}
\end{align*}

\subsection{\VAT{} Value Estimate}
The \VAT{} estimator gives an unbiased estimate of the expected value
$\mathop{\mathbb{E}}_z[v_p(z)]$. If we use partitions $\mathcal{H}$
and $\mathcal{W}$ as described above, and are given an observation of
a terminal state $z \in W \in \mathcal{W}$, the value estimate is
\begin{align}
\label{eqn:vat}
\text{\VAT{}}(z) = \frac{\sum_{z' \in W} \pi_{P_a}(z')
v_p(z')}{\sum_{z' \in W}\pi_{P_a}(z')} + \sum_{H \in \mathcal{H}}
k_H(z)
\end{align}

Note that there is a subtle difference between \VAT{} and a simple
combination of imaginary observations and an extended MIVAT framework
using player strategy information to add control variates for
actions. Using an extended MIVAT plus imaginary observations, we would
consider the expected MIVAT value estimate across all terminal
histories compatible with the observed terminal state. In \VAT{}, for
each correction term we would consider all histories compatible with
the state at that decision point.

As a concrete example of the difference, consider the game used in
Figure~\ref{fig:vat-diagram}. MIVAT with imaginary observations would
only consider private cards for player 1 that do not conflict with the
opponent's \textbf{K$\spadesuit$} or the public
card \textbf{J$\heartsuit$}, even when computing the
$\mathop{\mathbb{E}}[u(c)]-u(\textbf{J$\heartsuit$})$ control variate
term for the public card. In contrast, \VAT{}
considers \textbf{J$\heartsuit$} as a possible player card for the
term.

\section{Unbiased Value Estimate}
It is desirable to have an unbiased value estimate for games, so that
players can not improve their estimated value by changing their
strategy to fit the estimation technique. We prove that \VAT{} is
unbiased. The value estimate $\text{AIVAT}(z)$ in
Equation~\ref{eqn:vat} is a sum of two parts. The fraction in the
first part is an unbiased estimator based on imaginary
observations~\cite{08icml-ispoker}, so we only need to show that the
sum of all $k_H$ terms has an expected value of $0$.

\begin{lemma}
\label{lemma:average-zero}
$\forall H \in \mathcal{H}$ $\mathop{\mathbb{E}}_{z \in Z} [ k_H(z) ] = 0$
\end{lemma}
\begin{proof}
Consider an arbitrary $H \in \mathcal{H}$. Let $Z(H) = \{ z \in Z |
\exists h \in H, h \sqsubset z\}$ be the set of terminal states
passing through $H$. Expanding definitions, using property 1 of
$\mathcal{H}$ and multiplying by $\pi_{P_o}(H)/\pi_{P_o}(H)=1$ we get
\begin{align*}
\mathop{\mathbb{E}}_{z \in Z} [ k_H(z) ] = \sum_{z \in Z} \pi(z)k_H(z) = \sum_{z \in Z(H)} \pi(z)k_H(z) \\
= \sum_{z \in Z(H)} \pi(z) \frac{\pi_{P_o}(H)}{\pi_{P_o}(H)} \frac{\sum_{a \in A(H)} \sum_{h \in H}\pi_{P_a}(h \cdot a)u_h(a)}{\sum_{h \in H} \pi_{P_a}(h)} \\
- \sum_{z \in Z(H)} \pi(z) \frac{\pi_{P_o}(H)}{\pi_{P_o}(H)} \frac{\sum_{h \in H}\pi_{P_a}(h \cdot a_\text{O})u_h(a_\text{O})}{\sum_{h \in H} \pi_{P_a}(h \cdot a_\text{O})} \\
\end{align*}
Using $\pi_{P_o}(h)\pi_{P_a}(h) = \pi(h)$
\begin{align*}
= \sum_{z \in Z(H)} \pi(z)\frac{\sum_{a \in A(H)} \sum_{h \in H} \pi(h \cdot a)u_h(a)}{\sum_{h \in H} \pi(h)} \\
- \sum_{z \in Z(H)} \pi(z)\frac{\sum_{h \in H}\pi(h \cdot a_\text{O})u_h(a_\text{O})}{\sum_{h \in H} \pi(h \cdot a_\text{O})} \\
\end{align*}
Using $\sum_{z,h \sqsubset z} \pi(z) = \pi(h)$ and $\sum_{z,h \cdot a \sqsubset z} \pi(z) = \pi(h \cdot a)$
\begin{align*}
= \sum_{h' \in H} \pi(h')\frac{\sum_{a \in A(H)} \sum_{h \in H}\pi(h \cdot a)u_h(a)}{\sum_{h \in H} \pi(h)} \\
- \sum_{h' \in H} \sum_{a \in A(h')} \pi(h' \cdot a)\frac{\sum_{h \in H}\pi(h \cdot a)u_h(a)}{\sum_{h \in H} \pi(h \cdot a)} \\
\end{align*}
Using property 3 of $\mathcal{H}$
\begin{align*}
= \sum_{h' \in H} \pi(h')\frac{\sum_{a \in A(H)} \sum_{h \in H}\pi(h \cdot a)u_h(a)}{\sum_{h \in H} \pi(h)} \\
- \sum_{a \in A(H)} \sum_{h' \in H} \pi(h' \cdot a)\frac{\sum_{h \in H}\pi(h \cdot a)u_h(a)}{\sum_{h \in H} \pi(h \cdot a)} \\
= \sum_{a \in A(H)} \sum_{h \in H} \pi(h \cdot a)u_h(a)
- \sum_{a \in A(H)} \sum_{h \in H} \pi(h \cdot a)u_h(a) \\
= 0
\end{align*}
Because the expected value is 0 for an arbitrary $H$, the expected
value is 0 for the sum of all $H \in \mathcal{H}$.
\end{proof}

\begin{theorem}
$\mathop{\mathbb{E}}_{z \in Z} [\sum_{H \in \mathcal{H}} k_H(z) ] = 0$
\end{theorem}
\begin{proof}
This immediately follows from Lemma~\ref{lemma:average-zero}, as the
expected value of a sum of terms is the sum of the expected values of
the terms, which are all $0$.
\end{proof}

\section{Experimental Results}
We demonstrate the effectiveness of \VAT{} in two poker games,
\leduc{} and heads-up no-limit \texas{} (HUNL). Both \leduc{} and
HUNL have a convenient structure where all actions are public, and
there is a mix of chance events in the form of completely public board
cards and completely private hole cards. The uncomplicated structure
leads to a clear choice for the partition $\mathcal{H}$. Each
$H \in \mathcal{H}$ has states with identical betting, public board
cards, and private hole cards for any players in $P_o$.

In all experiments the value functions $u_h(a)$ are self-play values,
generated by solving the game to find a Nash equilibrium strategy
using a variant of the Monte Carlo CFR
algorithm~\cite{09nips-mccfr}. For each player $p_x$ and partition
$H$, we save the average observed values for opponent $p_y$ across all
iterations, giving us a value $w_H(a) \approx \sum_{h \in
H} \pi_{p_x}(h \cdot a) \mathop{\mathbb{E}}[v_{p_y}(h)] / \sum_{h \in
H} \pi_{p_x}(h \cdot a)$. $w_H(a)$ is an expected self-play value for
$p_y$ at $H$, given the probability distribution of hands for $p_x$
that reach $H$ and play $a$. Because we are playing a zero-sum game
and $v_{p_x}(h)=-v_{p_y}(h)$, we can use $u_h(a) = -w_{H}(a)$ $\forall
h \in H$. In HUNL, which is too large to solve directly, we solve a
very small abstraction of the game~\cite{03ijcai-psopti,Ganzfried14}
with only 8 million information sets, which gives us a rough estimate
of $w_H(a)$ that is identical across many partitions of HUNL states.

Poker is played in an alternating fashion, where agents take turns
playing in different positions. Let us say we have two agents, $x$ and
$y$. In poker, in odd-numbered games (starting at game 1) we would
have $x$ as player 1 and $y$ as player 2, and in even-numbered games
we would have $y$ as player 1 and $x$ as player 2. For the
experiments, we model this as an extended game where there is an
initial 50/50 chance event that assigns a position to the agent, along
with a \VAT{} correction term for the position.

All experiments will compare \VAT{} value estimates with the
unmodified game values from counting chips, the MIVAT value estimate,
and the combination of MIVAT and imaginary observations using the
strategy for agent $x$ (MIVAT+IO$_x$). Because poker is a zero-sum
game, it is sufficient to present results from the point of view of
agent $x$.

\subsection{\Leduc{}}
The small size of \leduc{} lets us test both the case where $P_a$ only
contains one non-chance player, as well as the full-knowledge case
where $P_a = P$. \VAT{} and chip count results are generated from
observations of 100,000 games. All of the numbers are in units of
chips, where \leduc{} has a 1 chip ante, and 2 chip and 4 chip bets in
the first and second rounds, respectively.

Figure~\ref{fig:leduc_selfplay} looks at self-play, where both $x$ and
$y$ play the same Nash equilibrium that was used to generate
$u_h(a)$. The true expected value for player $x$ is 0. Because we are
using value functions computed from their self-play, this experiment
represents a best-case situation. With knowledge of both player's
strategies, the only remaining variance comes from noise in the
$u_h(a)$ value function that arises from the sampling and averaging
used in the MCCFR computation.

\begin{figure}[!ht]
\centering
\begin{tabular}{|l||r|r|}
\hline
Estimator & $\bar{v}_x$ & $SD(v_x)$ \\
\hline
chips & 0.01374 & 3.513 \\
MIVAT & 0.00448 & 2.327 \\
MIVAT+IO$_x$ & 0.00987 & 1.928  \\
$P_a=\{p_c, x\}$ & -0.00009 & 0.00643 \\
$P_a=\{p_c, x, y\}$ & -0.00001 & 0.00377 \\
\hline
\end{tabular}
\caption{Value estimates for self-play in \leduc{}}
\label{fig:leduc_selfplay}
\end{figure}

With knowledge of both player's strategies, we reduce the per-game
standard deviation of the estimated player value by a little less than
99.9\%. This situation might be unlikely in practice, but does
demonstrate that the \VAT{} computation correctly shifts every
observed outcome to the expected player value, given full correct
information. Surprisingly, the one-sided evaluation where we use only
one player's strategy still reduces the standard deviation by
99.8\%. Using MIVAT or MIVAT+IO$_x$, we only see a 33.8\% and 45.1\%
reduction, respectively.

Moving away from the best-case situation, Figure~\ref{fig:leduc_asymm}
looks at games where $x$ is the same Nash equilibrium from above, and
$y$ is an agent that randomly calls or raises. Given these strategies,
the true expected value for player $x$ is 0.69358.

\begin{figure}[!ht]
\centering
\begin{tabular}{|l||r|r|}
\hline
Estimator & $\bar{v}_x$ & $SD(v_x)$ \\
\hline
chips & 0.71673 & 5.761 \\
MIVAT & 0.68932 & 4.412 \\
MIVAT+IO$_x$ & 0.69968 & 4.295 \\
$P_a=\{p_c, x\}$ & 0.69050 & 1.437 \\
$P_a=\{p_c, x, y\}$ & 0.68698 & 1.782 \\
$P_a=\{p_c, y\}$ & 0.69614 & 2.983 \\
\hline
\end{tabular}
\caption{Value estimates for dissimilar strategies in \leduc{}}
\label{fig:leduc_asymm}
\end{figure}

Using the call/raise strategy for $y$ demonstrates that the amount of
variance reduction does depend on how well the value functions
estimates the true expected value of a situation. We used value
functions which encode self-play values for $x$, and while $y$ is
sufficiently similar to $x$ that the true values are still positively
correlated with the estimated values for both players, they are no
longer an almost-perfect match. Despite the strategic mismatch,
using \VAT{} we see a reduction in the standard deviation of 48\% to
75\% compared to the basic chip-count estimate. All of the \VAT{}
estimators outperform the 25\% reduction using MIVAT plus imaginary
observations.

\subsection{No-limit \Texas{}}
The game of HUNL better represents a potential real-world
application. The game is commonly played, it is too large to easily
compute exact expected values directly even when the strategy of both
agents is known, average win rate is a statistic of interest to
players and observers, and the high per-game variance of outcomes
obscures the win rate even after hundreds of thousands of hands.

The variant of HUNL that we use has a small blind of 1 chip and big
blind of 2 chips, and each player has 200 chips (\ie 100 big blinds.)
Due to the large branching factor of chance events, we can only
present results for \VAT{} analysis using the strategy of one
agent. All results are generated from observations of 1 million
games.

We start by looking at self-play, using a low-quality Nash equilibrium
approximation for both players $x$ and $y$. The value functions
$u_h(a)$ are generated using this same weak
approximation. Figure~\ref{fig:texas_selfplay} gives the results for
the different estimation methods. The true expected value for $x$ is
0.

\begin{figure}[!ht]
\centering
\begin{tabular}{|l||r|r|}
\hline
Estimator & $\bar{v}_x$ & $SD(v_x)$ \\
\hline
chips & 0.03871 & 25.962 \\
MIVAT & 0.02038 & 21.293 \\
MIVAT+IO$_x$ & 0.02596 & 16.073 \\
$P_a=\{p_c, x\}$ & 0.00186 & 8.095 \\
\hline
\end{tabular}
\caption{Value estimates for self-play in HUNL}
\label{fig:texas_selfplay}
\end{figure}

In Figure~\ref{fig:texas_asymm} we look at games where $x$ uses the
same low-quality approximation of a Nash equilibrium, and $y$ is a
much stronger agent using a high-quality approximation of a Nash
equilibrium. The value functions $u_h(a)$ are still generated using
the low-quality approximation. The true expected value for player $x$
is not known.

\begin{figure}[!ht]
\centering
\begin{tabular}{|l||r|r|}
\hline
Estimator & $\bar{v}_x$ & $SD(v_x)$ \\
\hline
chips & -0.10017 & 26.308 \\
MIVAT & -0.11565 & 21.546 \\
MIVAT+IO$_x$ & -0.11297 & 16.051 \\
$P_a=\{p_c, x\}$ & -0.10971 & 8.301 \\
\hline
\end{tabular}
\caption{Value estimates for dissimilar strategies in HUNL}
\label{fig:texas_asymm}
\end{figure}

In both experiments, we see a 39\% reduction in the standard deviation
when using MIVAT with imaginary observations, and a bit more than a
68\% reduction using \VAT{}. It must be noted that our value function
could be improved, as the 18\% reduction for MIVAT in this experiment
does not match the 23\% improvement previously demonstrated using
values learned from data~\cite{09white-mivat}. The small abstract game
used to generate the value functions does not do a good job of
understanding the consequences of cards being dealt, as it can not
distinguish most card situations. Despite this handicap, the
full \VAT{} estimator still significantly improves on the state of the
art for low-variance value estimators for imperfect information games.

\section{Conclusions}
We introduce a technique for value estimation in imperfect information
games that extends and combines existing techniques. \VAT{} uses
heuristic value functions, knowledge of game structure, and knowledge
about player strategies to both add a control variate term for chance
and player decisions, and to average over multiple possible outcomes
given a single observation. We prove \VAT{} is unbiased, and
demonstrate that with (almost) perfect value functions we see (almost)
complete elimination of variance. Even with imprecise value functions,
we show variance reduction in a real-world game that significantly
exceeds existing techniques. \VAT{}'s three times reduction in
standard deviation allows us to achieve the same statistical
significance with ten times less data. A factor of ten is substantial:
for problems with limited data, like human play against bots, ten
times as many games could be the distinction between practical and
impractical.

\bibliographystyle{named}
\bibliography{paper}
\end{document}